\documentclass[runningheads]{llncs}

\usepackage{array}
\usepackage{adjustbox}
\usepackage{bm}
\usepackage{cite}
\usepackage{soul}
\usepackage{url}
\usepackage[utf8]{inputenc}
\usepackage[small]{caption}
\usepackage{graphicx}
\usepackage{amsmath}
\usepackage{latexsym,amssymb}
\usepackage{booktabs}
\usepackage{empheq}
\usepackage{fancybox}
\usepackage{subcaption}
\usepackage{algorithm2e}
\usepackage[dvipsnames]{xcolor}
\usepackage[english]{babel}
\usepackage{xspace}
\usepackage{setspace}
\usepackage{stmaryrd}
\usepackage{enumitem}
\usepackage{verbatim}
\usepackage{karnaugh-map}

\let\llncssubparagraph\subparagraph
\let\subparagraph\paragraph
\usepackage{titlesec}
\let\subparagraph\llncssubparagraph

\usepackage[pdftex%
,colorlinks=true       
,linkcolor=MidnightBlue        
,citecolor=MidnightBlue        
,filecolor=black       
,urlcolor=MidnightBlue         
,linktoc=page           
,plainpages=false]{hyperref}
\addto\extrasenglish{%

}

\graphicspath{{plots/}}
\DeclareGraphicsExtensions{.pdf}

\title{SAT-Based Rigorous Explanations \\for Decision Lists
\thanks{
This work is supported by the AI Interdisciplinary Institute ANITI,
funded by the French program ``Investing for the Future - PIA3'' under
Grant agreement n$^{o}$ ANR-19-PI3A-0004.
}}

\author{
  Alexey Ignatiev\inst{1} \and
  Joao Marques-Silva\inst{2}
}
\authorrunning{A. Ignatiev and J. Marques-Silva}

\institute{
  Monash University, Melbourne, Australia \\
  \href{mailto:alexey.ignatiev@monash.edu}{\texttt{alexey.ignatiev@monash.edu}} \and
  IRIT, CNRS, Toulouse, France \\
  \href{mailto:joao.marques-silva@irit.fr}{\texttt{joao.marques-silva@irit.fr}}
}

\newcommand{\todoF}[2]{}

\newcommand{\fml}[1]{{\mathcal{#1}}}
\newcommand{\mfrk}[1]{{\mathfrak{#1}}}

\newcommand{\tn}[1]{\textnormal{#1}}
\newcommand{\mbf}[1]{\ensuremath\mathbf{#1}}
\newcommand{\mbb}[1]{\ensuremath\mathbb{#1}}

\newcommand{\ddp}{\tn{D}^{\tn{P}}}

\definecolor{gray}{rgb}{.4,.4,.4}
\definecolor{midgrey}{rgb}{0.5,0.5,0.5}
\definecolor{middarkgrey}{rgb}{0.35,0.35,0.35}
\definecolor{darkgrey}{rgb}{0.3,0.3,0.3}
\definecolor{darkred}{rgb}{0.7,0.1,0.1}
\definecolor{midblue}{rgb}{0.2,0.2,0.7}
\definecolor{darkblue}{rgb}{0.1,0.1,0.5}
\definecolor{defseagreen}{cmyk}{0.69,0,0.50,0}

\DeclareMathOperator*{\nentails}{\nvDash}
\DeclareMathOperator*{\entails}{\vDash}

\DeclareMathOperator*{\limply}{\rightarrow}

\newcommand{\mailtodomain}[1]{\href{mailto:#1@ciencias.ulisboa.pt}{\texttt{\nolinkurl{#1}}}}

\titleformat{\paragraph}[runin]
            {\normalfont\normalsize\bfseries}{\theparagraph}{0.4em}{}

\newcommand{\frmeq}[1]{\begin{empheq}[box={\fboxsep=1.5pt\doublebox}]{flalign*}#1\end{empheq}}

\begin{document}

\maketitle

%
\begin{abstract}
  Decision lists (DLs) find a wide range of uses for classification
  problems in Machine Learning (ML), being implemented in a number of
  ML frameworks.
  DLs are often perceived as interpretable. However, building on
  recent results for decision trees (DTs), we argue that
  interpretability is an elusive goal for some DLs. As a result, for
  some uses of DLs, it will be important to compute (rigorous)
  explanations. 
  Unfortunately, and in clear contrast with the case of DTs, this
  paper shows that computing explanations for DLs is computationally
  hard. Motivated by this result, the paper proposes propositional
  encodings for computing abductive explanations (AXps) and
  contrastive explanations (CXps) of DLs. Furthermore, the paper
  investigates the practical efficiency of a MARCO-like approach for
  enumerating explanations.
  The experimental results demonstrate that, for DLs used in practical
  settings, the use of SAT oracles offers a very efficient solution,
  and that complete enumeration of explanations is most often
  feasible.
\end{abstract}
%


\section{Introduction} \label{sec:intro}

Decision lists (DLs)~\cite{rivest-ml87} find a wide range of uses for
classification problems in Machine Learning
(ML)~\cite{clark-ml89,cohen-ijcai93,cohen-icml95,cohen-aaai99,rudin-aistats15,rudin-jmlr17,rudin-kdd17,rudin-icml17,cohen-nips17,rudin-mpc18,rudin-aistats18},
being implemented in a number of ML frameworks
(e.g.~\cite{witten-jmlr10,demsar-jmlr13}).
DLs can be viewed as ordered rules, and so are often perceived as
interpretable\footnote{%
  Interpretability is a subjective concept, for which no rigorous
  accepted definition exists~\cite{lipton-cacm18}.
  As clarified later in the paper, for a given pair ML model and
  instance, we equate interpretability with how succinct is the
  justification for the model's prediction.}.
This explains in part the recent interest in
DLs~\cite{rudin-aistats15,rudin-jmlr17,rudin-kdd17,rudin-icml17,cohen-nips17,rudin-mpc18,rudin-aistats18},
most of which is premised on the interpretability of DLs.
However, building on recent results for decision trees
(DTs)~\cite{iims-corr20}, which demonstrate the possible
non-interpretability of DTs when representing specific functions, we
show that interpretability can also be an elusive goal for some DLs.
As a result, and for some concrete applications of DLs, it is
important to compute (rigorous) explanations.
%

Explanations can be broadly categorized into
heuristic~\cite{guestrin-kdd16b,lundberg-nips17,guestrin-aaai18} and
non-heuristic \cite{darwiche-ijcai18,inms-aaai19}. Recent work has
provided extensive evidence regarding the lack of quality of heuristic
explanation approaches~\cite{nsmims-sat19,inms-corr19,lukasiewicz-corr19,ignatiev-ijcai20,lakkaraju-aies20a,lakkaraju-aies20b}.
Non-heuristic (or rigorous) approaches for computing explanations can
be organized into
abductive
(AXp)~\cite{darwiche-ijcai18,inms-aaai19,inms-nips19,darwiche-ecai20}
and contrastive (CXp)~\cite{miller-aij19,inams-aiia20}. (Abductive
explanations are also referred to as
PI-explanations~\cite{darwiche-ijcai18} (i.e.\ prime implicant
explanations), since these represent subset-minimal sets of feature
value pairs that are sufficient for a prediction.)
Most work on rigorous explanations either exploits knowledge
compilation
approaches~\cite{darwiche-ijcai18,darwiche-aaai19,darwiche-ecai20,marquis-kr20},
or approaches based on iterative calls to some oracle for NP
(e.g.\ SAT, SMT, MILP, etc.)~\cite{inms-aaai19,inms-nips19,inams-aiia20}.
As a result, improvements to automated reasoning tools, can have a
profound impact on the deployment of rigorous explanation
approaches.

Furthermore, recent work proposed polynomial time algorithms for
finding explanations of a number of ML models, including
DTs~\cite{iims-corr20}, naive-Bayes classifiers~\cite{msgcin-nips20},
and also different knowledge representation
languages~\cite{marquis-kr20}.
Unfortunately, and in contrast with these recent tractability results,
this paper proves that finding one PI-explanation for a DL is NP-hard.

Motivated by the NP-hardness of finding explanations of DLs, the paper
proposes propositional encodings for computing abductive and
contrastive explanations of DLs. Furthermore, the paper investigates
the practical efficiency of a MARCO-like~\cite{lpmms-cj16} approach
for enumerating explanations.
The experimental results demonstrate that, for DLs used in practical
settings, the use of SAT oracles offers a very efficient solution, and
that complete enumeration of explanations is most often feasible.

The paper is organized as follows.
The notation and definitions used throughout the paper are introduced
in~\autoref{sec:prelim}.
\autoref{sec:approach} proves the NP-hardness of finding
rigorous explanations for DLs. In addition, this section develops a
propositional encoding for finding one AXp or one CXp, and briefly
overviews the online enumeration of explanations.
\autoref{sec:res} presents the experimental results.
The paper concludes in~\autoref{sec:conc}.

%


\section{Preliminaries} \label{sec:prelim}

\subsection{Propositional Satisfiability} \label{sec:prelim-sat}
Definitions standard in propositional satisfiability (SAT) and maximum
satisfiability (MaxSAT) solving are assumed~\cite{sat-handbook}.
In what follows, we will assume formulas to be propositional.
A conjunction of literals is referred to as \emph{term} while a
disjunction of literals is referred to as \emph{clause};
also note that a literal is either a Boolean variable or its negation.
Whenever convenient, terms and clauses are treated as sets of
literals.
A formula is said to be in \emph{conjunctive} or \emph{disjunctive}
\emph{normal form} (CNF or DNF, respectively) if it is a conjunction
of clauses or disjunction of terms, respectively.
Set theory notation will be also used with respect to CNF and DNF
formulas when necessary.

A \emph{truth assignment} $\mu$ is a mapping from the set of variables
to $\{0, 1\}$.
An assignment is said to satisfy a literal $l$ ($\neg{l}$, resp.) if
it maps variable $l$ to 1 (to 0, resp.).
A clause is said to be satisfied by assignment $\mu$ if $\mu$
satisfies at least one of its literals.
If for a CNF formula $\phi$ there exists an assignment $\mu$ that
satisfies all clauses of $\phi$, formula $\phi$ is referred to as
satisfiable and $\mu$ is its \emph{satisfying assignment} (or
\emph{model}).
In addition, we use the notation $\entails$ to denote
\emph{entailment}, i.e.\ $\phi_1\entails\phi_2$ if any model of
$\phi_1$ is also a model of $\phi_2$.

One of the central concepts in rigorous explainable AI
(XAI)~\cite{darwiche-ijcai18,inms-aaai19} is of prime implicants as
defined below.
\begin{definition}\label{def:prime}
  A term $\pi$ is an \emph{implicant} of formula $\phi$ if
  $\pi\entails\phi$.
  An implicant $\pi$ of $\phi$ is called \emph{prime} if none of the
  proper subsets $\pi'\subsetneq\pi$ is an implicant of~$\phi$.
\end{definition}

In the context of unsatisfiable formulas, the maximum satisfiability
(MaxSAT) problem is to find a truth assignment that maximizes the
number of satisfied clauses.
A number of variants of MaxSAT
exist~\cite[Chapters~23~and~24]{sat-handbook}.
Hereinafter, we are mostly interested in Partial (Unweighted) MaxSAT,
which can be formulated as follows.
The formula $\phi$ is represented as a conjunction of \emph{hard}
clauses $\fml{H}$, which must be satisfied, and \emph{soft} clauses
$\fml{S}$, which represent a preference to satisfy those clauses,
i.e.\ $\phi=\fml{H}\land\fml{S}$.
%
%
Therefore, the Partial MaxSAT problem consists in finding an
assignment that satisfies all the hard clauses and maximizes the total
number of satisfied soft clauses.
In the following, the concepts of minimal unsatisfiable subsets
(MUSes) and minimal correction subsets (MCSes) taking into account the
hard clauses $\fml{H}$ will also be helpful.
Concretely, consider unsatisfiable CNF formula
$\phi=\fml{H}\land\fml{S}$ with $\fml{H}$ and $\fml{S}$ defined as the
set of hard and soft clauses, respectively.
\begin{definition} \label{def:mus}
  A subset of soft clauses $\fml{M}\subseteq\fml{S}$ is a
  \emph{Minimal Unsatisfiable Subset} (MUS) iff $\fml{H}\cup\fml{M}$
  is unsatisfiable and
  $\forall_{\fml{M}'\subsetneq\fml{M}},\,\fml{H}\cup\fml{M}'$ is
  satisfiable.
\end{definition}

\begin{definition} \label{def:mcs}
  A subset of soft clauses $\fml{C}\subseteq\fml{S}$ is a {\em Minimal
  Correction Subset} (MCS) iff $\fml{H}\cup\fml{F}\setminus\fml{C}$ is
  satisfiable and
  $\forall_{\fml{C}'\subsetneq\fml{C}},\,\fml{H}\cup\fml{F}\setminus\fml{C}'$
  is unsatisfiable.
\end{definition}

MUSes and MCSes of a CNF formula are known to be related through the
minimal hitting set (MHS)
duality~\cite{reiter-aij87,lozinskii-jetai03,stuckey-padl05,liffiton-jar08},
which has been recently exploited in a number of practical
settings~\cite{liffiton-jar08,bacchus-cp11,iplms-cp15,imms-ecai16,imwms-ijcai19}
including XAI~\cite{inams-aiia20}.

\subsection{Classification Problems, Decision Lists, and Explanations}
\label{sec:mldef}

This section introduces definitions and notation related with
classification problems in ML, but also formal definitions of
explanations proposed in recent
work~\cite{darwiche-ijcai18,inms-aaai19}.

\paragraph{Classification problems.}
%
%
%
We consider a classification problem, characterized by a set of
(categorical) features $\fml{F}=\{1,\ldots,m\}$, and by a set of
classes $\fml{K}=\{c_1,\ldots,c_K\}$.
Each feature $j\in\fml{F}$ is characterized by a domain $D_i$.  As a
result, feature space is defined as
$\mbb{F}={D_1}\times{D_2}\times\ldots\times{D_m}$. A specific point in
feature space is represented by $\mbf{v}=(v_1,\ldots,v_m)$. A point
$\mbf{v}$ in feature space denotes an \emph{instance} (or an
\emph{example}). Moreover, we use $\mbf{x}=(x_1,\ldots,x_m)$ to denote
an arbitrary point in feature space. In general, when referring to the
value of a feature $j\in\fml{F}$, we will use a variable $x_j$, with
$x_j$ taking values from $D_j$. (To keep the notation simple, we opt
not to introduce an assignment function, mapping each feature $j$ to
some value in $D_j$.)
For simplicity, throughout this paper we will restrict $\fml{K}$ to
two classes, i.e.\ $\fml{K}=\{\oplus,\ominus\}$. However, most of the
ideas described in this document also apply in the more general case
of $\fml{K}$ with more than two elements; the general case of
non-binary classification is also considered in the experimental
results presented in~\autoref{sec:res}.
(In settings where $\fml{K}=\{\oplus,\ominus\}$, we will also equate
$\oplus$ with 1, and $\ominus$ with 0.)

A classifier implements a \emph{total classification function}
$\tau:\mbb{F}\to\fml{K}$.
%
%
In some settings, e.g.\ when computing explanations, it will be
convenient to represent the classification function as a
\emph{decision predicate} $\tau_c:\mbb{F}\to\{0,1\}$, parametrized by
some fixed class $c\in\fml{K}$, and such that
$\forall(\mbf{x}\in\mbb{F}).\tau_c(\mbf{x})\leftrightarrow(\tau(\mbf{x})=c)$.

\begin{figure}[t]
  \begin{subfigure}[b]{0.3575\textwidth}
    \small
    \begin{minipage}{\textwidth}
      \centering
      $\tau(\mbf{x})=\left\{
      \begin{array}{lcl}
        \oplus & \quad & \tn{if~~}[2x_1-x_2>1]
        \\[8pt]
        \ominus & \quad & \tn{if~~}[2x_1-x_2\le1]
        \\
      \end{array}
      \right.
      $
      \vspace{11pt}
    \end{minipage}
    \caption{Example linear classifier} \label{fig:lc01}
  \end{subfigure}
  \begin{subfigure}[b]{0.314\textwidth}
    \begin{center}
      \scalebox{0.825}{

        \begin{minipage}{1.18\textwidth}
        \frmeq{
          %
          %
          \begin{array}{lllcl}
            \tn{R$_{0}$:} & \tn{IF}      & x_1 & \tn{THEN} & \oplus \\
            \tn{R$_{1}$:} & \tn{ELSE IF} & x_2 & \tn{THEN} & \oplus \\
            \tn{R$_{\tn{\sc{def}}}$:} & \tn{ELSE} & \quad\quad & \tn{THEN} & \ominus \\
          \end{array}
        }
        \end{minipage}
      }
    \end{center}
    \caption{Example decision list} \label{fig:dl01}
  \end{subfigure}
  \begin{subfigure}[b]{0.314\textwidth}
    \begin{center}
      \scalebox{0.825}{
        \begin{minipage}{1.16\textwidth}
        \frmeq{
          %
          %
          \begin{array}{lllcl}
            \tn{R$_{0}$:} & \tn{IF}      & x_1 & \tn{THEN} & \oplus \\
            \tn{R$_{1}$:} & \tn{IF}      & x_2 & \tn{THEN} & \oplus \\
            \tn{R$_{2}$:} & \tn{IF} & \neg{x_1}\land\neg{x_2} & \tn{THEN} & \ominus \\
          \end{array}
        }
        \end{minipage}
      }
    \end{center}
    \caption{Example decision set} \label{fig:ds01}
  \end{subfigure}
  \caption{Example classifiers}
\end{figure}

\begin{example} \label{ex:lc01}
  To illustrate the definitions above, we consider a very simple
  linear classifier, defined as follows.
  Let $\fml{F}=\{1,2\}$, with $D_1=D_2=\{0,1,2\}$, and let
  $\fml{K}=\{\ominus,\oplus\}$.
  As a result, feature space is given by
  $\mbb{F}=\{0,1,2\}\times\{0,1,2\}$. Furthermore, the classification
  function associated with the classifier is shown
  in~\autoref{fig:lc01}. Concretely, the prediction is $\oplus$ if
  $2x_1-x_2>1$, and it is $\ominus$ otherwise.
  %
  \qed
\end{example}

%

\paragraph{Decision lists (DLs) \& decision sets (DSs).}
\label{par:dlds}
A \emph{rule} is of the form
``IF~\textsf{antecedent}~THEN~\textsf{prediction}'', where the
\textsf{antecedent} is of the form
$\bigwedge$\textsf{feature-literals}.
The interpretation of a rule is that if the \textsf{antecedent} is
consistent (i.e.\ all the literals are true), then the rule
\emph{fires} and the prediction is the one associated with the rule.
A decision list (DL)~\cite{rivest-ml87} is an \emph{ordered} list of
rules, whereas a decision set (DS)~\cite{clark-ewsl91,leskovec-kdd16}
is an \emph{unordered} list of rules.
%

Throughout the paper, we will consider ordered
sets of rule indices $\mfrk{R}=\{1,\ldots,R\}$, such that for
$i\in\mfrk{R}$, we will use $\mfrk{c}$, $\mfrk{l}$ and $\mfrk{o}$ to
denote, respectively, the class associated with rule $i$, the set of
literals associated with rule $i$ and
the order of rule $i$.

\begin{example} \label{ex:dl01}
  Consider another classifier.
  Let $\fml{F}=\{1,2\}$, with $D_1=D_2=\{0,1\}$, and so
  $\mbb{F}=\{0,1\}\times\{0,1\}$.
  The decision list for the classifier is shown in~\autoref{fig:dl01}
  while an equivalent decision set is shown in~\autoref{fig:ds01}.
  The classification function for the DL can be represented as
  follows:
  \[
  \tau(\mbf{x}) = \left\{
  \begin{array}{lcl}
    \oplus  & \quad &
    \tn{if~~}[(x_1)\lor(\neg{x_1}\land{x_2})]\\[10pt]
    \ominus & \quad & \tn{if~~}[(\neg{x_1}\land\neg{x_2})]\\
  \end{array}
  \right.
  \]
  (Note how the lack of order in DS rules results in a simpler
  classifier representation $\tau(\mbf{x})$ for class $\oplus$, e.g.\
  it can be explicitly represented as $x_1 \lor x_2$ since rules
  \tn{R$_{0}$} and \tn{R$_{1}$} are unordered in the decision set of
  \autoref{fig:ds01}.)
  \qed
\end{example}

Note that following the standard convention, we will \emph{always}
assume that DLs have a \emph{default rule}, with no literals, that
\emph{fires} when for all the preceding rules, the conjunction of
literals associated with that rule is inconsistent.
An example default rule for the DL shown in \autoref{ex:dl01} is
marked as \tn{R$_{\tn{\sc{def}}}$}.

\paragraph{Interpretability \& explanations.}
Interpretability is generally accepted to be a subjective concept,
without a formal definition~\cite{lipton-cacm18}. In this paper we
measure interpretability in terms of the overall succinctness of the
information provided by an ML model to justify a given prediction.
We say that a model is \emph{not} interpretable if for some instance,
the justification of a prediction is arbitrarily larger (on the number
of features) than a rigorous explanation (which we define next).
Moreover, and building on earlier work, we equate explanations with
the so-called
PI-explanations~\cite{darwiche-ijcai18,inms-aaai19,darwiche-ecai20,marquis-kr20},
i.e.\ subset-minimal sets of feature-value pairs that are sufficient
for the prediction.
More formally, given an instance $\mbf{v}\in\mbb{F}$, with prediction
$c\in\fml{K}$, i.e.\ $\tau(\mbf{v})=c$, a PI-explanation is a minimal
subset $\fml{X}\subseteq\fml{F}$ such that,
\begin{equation} \label{eq:axp}
\forall(\mbf{x}\in\mbb{F}).\bigwedge\nolimits_{j\in\fml{X}}(x_j=v_j)\limply(\tau(\mbf{x})=c)
\end{equation}
Another name for a PI-explanation is (a minimal/minimum)
\emph{abductive explanation} (AXp)~\cite{inms-aaai19,inams-aiia20}.
For simplicity, and depending on the context, we will use
\emph{PI-explanation} and the acronym \emph{AXp} interchangeably.

In a similar vein, we consider \emph{contrastive explanations}
(CXps)~\cite{miller-aij19,inams-aiia20}.
Contrastive explanation can be defined as a (subset-)minimal set of
feature-value pairs ($\fml{Y}\subseteq\fml{F}$) that suffice to
changing the prediction if they are allowed to take \emph{some
arbitrary} value from their domain.
Formally and as suggested in~\cite{inams-aiia20}, a CXp is defined as
a minimal subset $\fml{Y}\subseteq\fml{F}$ such that,
\begin{equation} \label{eq:cxp}
\exists(\mbf{x}\in\mbb{F}).\bigwedge\nolimits_{j\not\in\fml{Y}}(x_j=v_j)\land(\tau(\mbf{x})\not=c)
\end{equation}
(It is possible and simple to adapt the definition to target a
specific class $c'\not=c$.)
Moreover, building on the seminal work of Reiter~\cite{reiter-aij87},
recent work demonstrated a minimal hitting set relationship between
AXps and CXps~\cite{inams-aiia20}, namely each AXp is a minimal
hitting set (MHS) of the set of CXps and vice-versa.
%

%
%
For computing both kinds of explanations (AXps and CXps), we will work
with sets of features, aiming at finding minimal subsets.
It will also be helpful to describe explanations (concretely AXps) as
sets of literals.
As a result, starting from an instance $\mbf{v}$, we create a set of
literals $I_{\mbf{v}}=\{(x_j,v_j)|j\in\fml{F}\}$.
When  clear from the context, we will just use $I$ to denote the
literals of an instance.

An AXp $\fml{X}\subseteq\fml{F}$ can also be viewed as a conjunction
$\rho$ of a subset of the literals $I_{\mbf{v}}$ induced by the
instance $\mbf{v}$ that is \emph{sufficient} for the prediction.
Moreover, given a conjunction of literals $\rho$, we will associate a
predicate $\rho:\mbb{F}\to\{0,1\}$ (with the symbol duplication
deliberately aiming at simplifying the notation) to represent the
values taken by the conjunction of literals for each point $\mbf{x}$
in feature space.
As a result, we use $\rho\entails\tau_c$ to denote that $\rho$ is
sufficient for the prediction, i.e.\
\begin{equation}
\forall(\mbf{x}\in\mbb{F}).\rho(\mbf{x})\limply\tau_c(\mbf{x})
\end{equation}

We can also associate a conjunction of literals $\eta$ with each CXp,
such that the literals in $\eta$ are \emph{not} the literals specified
by the CXp, and such that the following condition holds,
\begin{equation}
  \exists(\mbf{x}\in\mbb{F}).\eta(\mbf{x})\land\neg\tau_c(\mbf{x})
\end{equation}
It should be noted that since a CXp is a minimal set of features, each
$\eta$ is a maximal set of literals such that there exists at least
one point in feature space such that the ML model predicts a class
other than $c$.


\begin{example}
  For the linear classifier of \autoref{ex:lc01}, let $\mbf{v}=(2,0)$,
  with prediction $\oplus$. In this case, the (only) AXp is
  $\fml{X}=\{1\}$, indicating that, as long as $x_1=2$, the value of
  the prediction is $\oplus$, independently of the value of $x_2$.
  Moreover, the AXp can also be represented by
  $\rho\triangleq(x_1=2)$. For this very simple example,
  $\fml{Y}=\{1\}$ is also a CXp. Indeed, if we allow feature 1 to
  take a value other than 2, then the assignment $\mbf{v}'=(0,0)$ will
  change the prediction. (More complex examples of CXps are studied
  later in the paper.)
  \qed
\end{example}

\begin{example}
  For the decision list of \autoref{ex:dl01}, let $\mbf{v}=(0,1)$,
  with prediction $\oplus$. In this case, the (only) AXp is
  $\fml{X}=\{2\}$, indicating that, as long as $x_2=1$, the value of
  the prediction is $\oplus$, independently of the value of $x_1$.
  Moreover, the AXp can also be represented by
  $\rho\triangleq(x_2=1)$. In this case, a CXp is also
  $\fml{Y}=\{2\}$. For example, the point in feature space
  $\mbf{v}=(0,0)$ will cause the prediction to change to $\ominus$.
  \qed
\end{example}

\begin{example} \label{ex:mhsdl01}
  To illustrate the hitting set duality relationship between AXps and
  CXps established in~\cite{inams-aiia20}, we consider a simple
  classifier represented as a decision list (DL) of three rules
  (including the default rule).
  Let $\fml{F}=\{1,2,3,4,5\}$, $D_i=\{0,1,2\}$, with $i=1,\dots,5$,
  and $\fml{K}=\{\ominus,\oplus\}$.
  Let the decision list be:
  \begin{center}
    \vspace*{-0.25cm}
    \begin{minipage}{0.5\textwidth}
      \frmeq{
        \small
        \begin{array}{llrl}
          \tn{R$_{0}$:} & \tn{IF} & x_1 = 1 \land x_2 = 1 & \tn{THEN
          } \ominus \\
          \tn{R$_{1}$:} & \tn{ELSE IF} & x_3\neq 1 & \tn{THEN }
          \oplus \\
          \tn{R$_{\tn{\sc{def}}}$:} & \tn{ELSE } & & \tn{THEN }
          \ominus \\
        \end{array}
      }
    \end{minipage}
  \end{center}
  We consider the instance $\mbf{v}=(1,1,1,1,1)$, which results in
  prediction $\ominus$.
  It is straightforward to see that, as long as $x_1=x_2=1$, then the
  prediction is $\ominus$.
  Also, it is less trivial (but still observable) that, as long as
  $x_3=1$, the prediction is guaranteed to be $\ominus$ as well.
  Moreover, it suffices to change the value of feature 3 and the value
  of \emph{either} feature 1 or feature 2 to change the prediction to
  $\oplus$, e.g.\ set $x_3=x_1=0$ or set $x_3=x_2=2$.
  As a result, we can conclude that the set of AXps is:
  $\mbb{X}=\{\{1,2\},\{3\}\}$, and the set of CXps is:
  $\mbb{Y}=\{\{1,3\},\{2,3\}\}$.
  Furthermore, from the minimal hitting set duality relationship
  between AXps and CXP's~\cite{inams-aiia20}, the sets in $\mbb{X}$
  are MHSes of the sets in $\mbb{Y}$ and vice-versa.
  (Clearly, we could follow the definitions and reach the same
  conclusions.)
  \qed
\end{example}



\section{Explaining Decision Lists} \label{sec:approach}

%
It is easy to see that just like DTs~\cite{iims-corr20}, DLs can also
exhibit redundancy in the literals used, and so the computation of
PI-explanations can be instrumental to conveying short explanations to
a human decision maker.

\begin{example} \label{ex:DL00}
  Consider a possible DL shown below for the function
  $f(x_1,\ldots,x_4)=(x_1\land x_2) \lor (x_3 \land x_4)$.
  (This DL is constructed by applying a ``direct translation'' of all
  the paths of the DT shown in~\cite[Figure 1b]{iims-corr20} from left
  to right into rules followed by appending a default rule predicting
  class $f=1$.)
  \begin{center} \vspace*{-0.35cm}
    \begin{minipage}{0.758\textwidth}
      \frmeq{
        \small
        \begin{array}{llrl}
          \tn{R$_{0}$:} & \tn{IF} & x_1 = 0 \land x_3 = 0 & \tn{THEN
          } f=0 \\
          \tn{R$_{1}$:} & \tn{ELSE IF} & x_1 = 0 \land x_3 = 1 \land
          x_4 = 0 & \tn{THEN } f=0 \\
          \tn{R$_{2}$:} & \tn{ELSE IF} & x_1 = 0 \land x_3 = 1 \land
          x_4 = 1 & \tn{THEN } f=1 \\
          \tn{R$_{3}$:} & \tn{ELSE IF} & x_1 = 1 \land x_2 = 0 \land
          x_3 = 0 & \tn{THEN } f=0 \\
          \tn{R$_{4}$:} & \tn{ELSE IF} & x_1 = 1 \land x_2 = 0 \land
          x_3 = 1 \land x_4 = 0 & \tn{THEN } f=0 \\
          \tn{R$_{5}$:} & \tn{ELSE IF} & x_1 = 1 \land x_2 = 0 \land
          x_3 = 1 \land x_4 = 1 & \tn{THEN } f=1 \\
          \tn{R$_{6}$:} & \tn{ELSE IF} & x_1 = 1 \land x_2 = 1 &
          \tn{THEN } f=1 \\
          \tn{R$_{\tn{\sc{def}}}$:} & \tn{ELSE } & & \tn{THEN } f=1 \\
        \end{array}
      }
    \end{minipage}
  \end{center}
  Consider a data instance $\mbf{v}=(1,0,1,1)$ and observe that rule
  \tn{R$_{5}$} fires the prediction $f=1$.
  Although rule \tn{R$_{5}$} has four literals, an AXp for instance
  $\mbf{v}$ is $(x_3=1)\land(x_4=1)$.
  Similarly, in practice one may expect examples of DLs s.t.\ AXps
  will be significanlty smaller than the rules that fire the
  corresponding predictions.

  This observation is confirmed by the experimental results
  in~\autoref{sec:res}, in that explanations can play an important
  role in understanding the predictions made by DLs.
  \qed
\end{example}


\subsection{DL Explainability} \label{sec:compl}

Perhaps surprisingly, whereas DTs can be explained in polynomial time,
DLs cannot.
This section proves a number of theoretical results related to
explainability of DLs.
Here we will be using the \emph{knowledge compilation} (KC)
map~\cite{darwiche-jair02}, which studied a wealth of queries on
knowledge representation languages.
We consider the concrete setting of classification, i.e.\ a language
$L$ denotes a classifier $\tau$ and a target prediction $c$.
Let us briefly define the queries of interest~\cite{darwiche-jair02}:
\begin{enumerate}
  \item \textbf{Satisfiability (SAT)}:
    \emph{if there exists a polynomial-time algorithm for deciding the
    satisfiability of $\tau(\mbf{x})=c$, i.e.\ to decide in
    polynomial time whether there exists $\mbf{x}\in\mbb{F}$ such
    that $\tau(\mbf{x})=c$.
    In the case of DLs, this problem will be referred to as \emph{DLSAT}.}
  \item \textbf{Implicant test (IM)}:
    \emph{if there exists a polynomial-time algorithm that decides
    whe\-ther a conjunction of literals $\rho$ is such that
    $\rho\entails\tau_c$, i.e.\
    $\forall(\mbf{x}\in\mbb{F}).\rho(\mbf{x})\limply\tau_c(\mbf{x})$.
    In the case of DLs, this problem will be referred to as \emph{DLIM}.}
\end{enumerate}
Similarly, we can define DNFSAT (which is trivially in $\tn{P}$) and
DNFIM (which is well-known to be in $\tn{P}$ only if
$\tn{P}=\tn{NP}$~\cite{darwiche-jair02}).

\begin{proposition} \label{prop:sat}
  DLSAT is \tn{NP}-complete.
\end{proposition}

\begin{proof}
  It is easy to see that the DLSAT is in NP. We simply guess an
  assignment to the features and check whether the prediction is the
  expected one according to the DL.
  To prove NP-hardness, the reduction of  CNFSAT to DLSAT is organized
  as follows:
  \begin{enumerate}[nosep]
  \item Consider a CNF formula $\phi$ with clauses
    $c_1,c_2,\ldots,c_m$.
  \item Let the variables in $\phi$ denote the features (w.l.o.g.\
    assume the features to be Boolean).
  \item Consider the negation of each clause $\neg{c_i}$ which
    represents a conjunction of literals $\bigwedge_{l_j\in c_i}
    \neg{l_j}$.
  \item For each $\neg{c_i}$, create a rule $\pi_i$ with antecedent
    $\bigwedge_{l_j\in c_i} \neg{l_j}$ and prediction $\ominus$.
  \item Create a default rule with prediction $\oplus$.
  \item Hence, formula $\phi$ is satisfiable if and only if there is
    an assignment to the features which results in prediction
    $\oplus$.
  \end{enumerate}
  The prediction is $\ominus$ if some clause $c_i$ is falsified, i.e.\
  $\neg{c_i}$ is satisfied (and hence rule $\pi_i$ fires).
  Otherwise, if all clauses are satisfied, and so all $\neg{c_i}$ are
  falsified, then the prediction is $\oplus$.
  \qed
\end{proof}

\begin{proposition} \label{prop:im}
  No polynomial-time algorithm exists for DLIM unless
  $\tn{P}=\tn{NP}$.
\end{proposition}

\begin{proof}
  We reduce DNFIM (i.e.\ IM for DNF) to DLIM, given that IM for DNF is
  well-known to be solvable in polynomial time only if
  $\tn{P}=\tn{NP}$~\cite{darwiche-jair02}.
  Let $\psi$ denote a DNF, with $k$ terms, i.e.\
  $\psi=t_1\lor\ldots\lor{t_k}$, and let $p$ denote a conjunction of
  literals.
  IM for DNF is to decide whether $p$ is an implicant of $\psi$, i.e.\
  $p\entails\psi$.
  The reduction of DNFIM to DLIM is organized as follows:
  \begin{enumerate}[nosep]
  \item For each conjunction of literals $t_i$ in $\psi$, create a
    rule with antecedent given by $t_i$, i.e.\ $\pi_i=t_i$, and
    prediction $\ominus$.
  \item The $(k+1)^{\tn{th}}$ rule is created as follows: the
    antecendent is $p$ and the prediction is $\oplus$.
  \item Finally, we add a default rule with prediction $\ominus$.
  \end{enumerate}
  As a result, the prediction will be $\oplus$ if and only if
  $p\land\bigwedge_{i\in[k]}(\neg{t_i})$ is satisfied, and so $p\nentails\psi$, in
  which case $p$ is not an implicant of $\psi$.
  \qed
\end{proof}

Given the above results, we can conclude the following.

\begin{proposition} \label{prop:expl}
  There is no polynomial-time algorithm for finding an AXp of a
  decision list unless $\tn{P}=\tn{NP}$.
\end{proposition}

\begin{proof}[sketch]
  If there was a polynomial-time algorithm for finding an AXp for a DL
  then we would be able to solve IM for DL in polynomial time.
  This would in turn imply that IM for DNF is solvable in polynomial
  time.
  \qed
\end{proof}

\paragraph{A Word on Decision Sets.}
Although decision sets are unordered (in contrast to DLs), this fact
does not simplify the computation of PI-explanations.
(In what follows, we assume that a DS implements a total
classification function, which is not the case in general due to the
issue of overlap~\cite{ipnms-ijcar18} --- otherwise, PI-explanations
would be ill-defined.)

\begin{proposition} \label{prop:ds}
  Finding an AXp for a DS is hard for $\ddp$.
\end{proposition}

\begin{proof}[sketch]
  It is known~\cite{ipnms-ijcar18} that decision sets can be
  associated with DNF formulas.
  It is also known~\cite{umans-tcad06} that finding a prime implicant
  (PI) of a DNF $D$ given a satisfying assignment $\mbf{v}$ is
  complete for $\ddp$.
  Given the aforementioned connection between DSs and DNFs, we show
  here that the above problem can be reduced to finding a
  PI-explanation of a DS.

  Let the terms in the DNF $D$ become the rules for prediction
  $\oplus$ in the corresponding DS.
  Also, let the default rule of the DS predict $\ominus$.
  Hence, a set of literals $\rho$ (contained in the literals induced
  by $\mbf{v}$) is a PI of the DNF $D$ iff $\rho$ is a PI-explanation
  for the DS prediction $\oplus$ given $\mbf{v}$.
  \qed
\end{proof}

\begin{remark}
  In the case of decision sets, it is also simple to observe that
  deciding whether a set of literals $\rho$ is an AXp is in $\ddp$.
  For that, one needs to prove first that the set of literals $\rho$
  entails prediction $\oplus$; this problem is clearly in coNP.
  Additionally, one also needs to prove subset-minimality of $\rho$,
  i.e.\ that removing any single literal from $\rho$ results in a
  subset of literals that does not entail the prediction $\oplus$.
  (We can consider $|\rho|$ sets of literals, each of which removes a
  literal from $\rho$ to get a set of literals $\rho_k$, and check
  that there are $|\rho|$ assignments such for each $\rho_k$ we get a
  different prediction $\ominus$.)
  The latter problem is in NP.
  Therefore and given \autoref{prop:ds}, we can establish
  $\ddp$-completeness of the decision version of finding a
  PI-explanation in the case of DSs.
\end{remark}

\paragraph{DLs vs.\ DTs and vs. DSs.}
The results of this section are somewhat surprising in terms of
comparing DTs with DLs and DSs.
On the one hand, satisfiability query is trivially in $\tn{P}$ for DTs
and DSs, but it is NP-complete for DLs.
On the other hand, AXps can be computed in polynomial time for
DTs\cite{iims-corr20}, but a polynomial-time algorithm for computing
AXps for DLs and DSs would imply $\tn{P}=\tn{NP}$.


\subsection{Explaining Arbitrary DLs with SAT} \label{sec:anydl}
%
When explaining decision lists,
one can use the work on computing rigorous
abductive~\cite{darwiche-ijcai18,inms-aaai19} and contrastive
explanations~\cite{inams-aiia20} for ML models.
This section describes a novel propositional encoding for DL
classifiers that can be exploited by the generic approach
of~\cite{inms-aaai19,inams-aiia20}.

Let $\mbf{v}$ denote a point in feature space with prediction
$c\in\fml{K}$.
Moreover, let the rule that fires on $\mbf{v}$ be $i\in\mfrk{R}$.
Note that for an arbitrary rule $k\in\mfrk{R}$ to fire, the following
constraint must hold true:
\begin{equation} \label{eq:dl01}
  \bigwedge_{\substack{r_j\in\mfrk{R}\\ 
      \mfrk{o}(j)<\mfrk{o}(k)}}\neg(\mfrk{l}(j))\land\mfrk{l}(k)
\end{equation}
Constraint~\eqref{eq:dl01} encodes the fact that the literals in all
the rules preceding rule $k$ must not fire and the rule $k$ must fire.
(Recall that $\mfrk{l}(i)$ represents the set of literals of rule
$i$).
This constraint is straightforward to clausify, i.e.\ convert to CNF.
Moreover, let $\varphi(i)$ denote the set of clauses resulting from
clausification of the constraint~\eqref{eq:dl01} for rule $i$ to fire.

Given a set of literals $\rho$, $\rho$ is an implicant of the decision
function associated with the DL (i.e.\ $\rho$ is an AXp) for the
instance $\mbf{v}$ and the corresponding prediction $\mfrk{c}(i)$ if:
\begin{equation} \label{eq:dl02}
\rho\entails
\bigvee_{\substack{j\in\mfrk{R}\\ 
    \mfrk{c}(j)=\mfrk{c}(i)}} \varphi(j)
\end{equation}
i.e.\ for any point $\mbf{x}$ in feature space, if $\rho(\mbf{x})$
holds true, then one of the rules predicting the same class
$\mfrk{c}(i)$ as rule $i$ must hold true as well.
Constraints \eqref{eq:dl01} and \eqref{eq:dl02} comprise the
propositional encoding that can be used in the framework
of~\cite{inms-aaai19} to compute one AXp for the prediction made by a
decision list for a given input instance.
Note that computing such an AXp $\rho$ is typically done by reducing
the initial set of literals $I_{\mbf{v}}$, which clearly entails the
right-hand side of \eqref{eq:dl02}, i.e.\
$I_{\mbf{v}}\entails\bigvee_{\substack{j\in\mfrk{R},\mfrk{c}(j)=\mfrk{c}(i)}}\varphi(j)$.
Also note that in practice it is convenient to negate this tautology
and instead deal with its negation, which is obviously unsatisfiable.
Following~\cite{inms-aaai19,inams-aiia20}, this enables one to apply
the well-developed apparatus for computing one AXp (resp. CXp) as an
MUS (resp. MCS) of the negated
formula~\cite{junker-aaai04,lms-sat04,liffiton-jar08,liffiton-cj09,bms-fmcad11,blms-aicom12,mshjpb-ijcai13,mpms-ijcai15,iplms-cp15,mipms-sat16,ijms-cj16},
but also for enumerating a given number of all AXps (resp. CXps)
through MUS (resp. MCS)
enumeration~\cite{ls-sat05,mlms-hvc12,pms-aaai13,lm-cpaior13,lpmms-cj16}.


\begin{example} \label{ex:enc}
  As mentioned above, when computing an AXp in the form
  of~\eqref{eq:dl02}, it is convenient to negate the tautology
  $I_{\mbf{v}}\entails\bigvee_{\substack{j\in\mfrk{R},\mfrk{c}(j)=\mfrk{c}(i)}}\varphi(j)$
  and instead work with unsatisfiable formula
  \begin{equation*} \label{eq:ex}
    I_{\mbf{v}} \land
  \bigwedge_{\substack{j\in\mfrk{R}\\ 
      \mfrk{c}(j)=\mfrk{c}(i)}} \neg{\varphi(j)}
  \end{equation*}
  Here, the left part $I_\mbf{v}$ of the conjunction serves as the set
  $\fml{S}$ of unit-size soft clauses, each represented by a literal
  assigning a value to a feature.
  This way AXps and CXps can be found as minimal subsets of $\fml{S}$
  (i.e.\ MUSes or MCSes, respectively), subject to the hard clauses
  $\fml{H}\triangleq
  \bigwedge_{\substack{j\in\mfrk{R},\mfrk{c}(j)=\mfrk{c}(i)}}
  \neg{\varphi(j)}$.
  Also observe that the negation $\neg\varphi(k)$ (recall that
  $\varphi(k)$ enforces rule $k$ to fire) constitutes the disjunction
  $$
    \neg{\mfrk{l}(k)}\lor \bigvee_{\substack{r_j\in\mfrk{R}\\
    %
    \mfrk{o}(j)<\mfrk{o}(k)}}\mfrk{l}(j)
  $$
  which enforces that either rule $k$ does not fire or one of the
  preceding rules fires.
  Also, to enforce that the default rule does not fire, we can simply
  require one of the non-default rules of the DL to fire.
  Finally, note that the hard clauses $\fml{H}$ encode the fact of
  \emph{misclassification}, which is clearly impossible when the input
  instance $I_\mbf{v}$ is given as the soft clauses $\fml{S}$, thus
  making formula $\fml{H}\land\fml{S}$ unsatisfiable.

  Now, consider the DL from \autoref{ex:DL00} and recall that rule
  \tn{R$_{5}$} fires prediction $f=1$ for the instance
  $\mbf{v}=(1,0,1,1)$.
  As prediction $f=1$ is represented by rules \tn{R$_{2}$},
  \tn{R$_{5}$}, \tn{R$_{6}$}, \tn{R$_{\tn{\sc{def}}}$}, our hard clauses
  $\fml{H}$ must enforce that none of them fires.
  Given the above, the hard clauses $\fml{H}$ are formed by
  $$
  \fml{H}=\left\{
    \def\arraystretch{1.5}
    \begin{array}{ll}
    \begin{array}{llr}
      \neg\varphi(2) & \triangleq & \left[\neg\mfrk{l}(2) \lor \bigvee_{j=0}^1 \mfrk{l}(j)\right]; \\
      \neg\varphi(6) & \triangleq & \left[\neg\mfrk{l}(6) \lor \bigvee_{j=0}^5 \mfrk{l}(j)\right]; \\
    \end{array}
    &
    \begin{array}{llr}
      \neg\varphi(5) & \triangleq & \left[\neg\mfrk{l}(5) \lor \bigvee_{j=0}^4 \mfrk{l}(j)\right]; \\
      \neg\varphi_{\tn{\sc{def}}} & \triangleq & \left[\bigvee_{j=0}^6 \mfrk{l}(j)\right]\hspace*{4.5pt} \\
    \end{array}
    \end{array}
  \right\}
  $$
  Here, CNF encoding of terms $\mfrk{l}(j)$ is omitted as it is
  trivial to obtain.
  \qed
\end{example}

As can be observed in~\autoref{ex:enc}, the propositional encoding
described in this section targets simplicity and for this reason it
exhibits redundancy, e.g.\ expressions
$\bigvee_{j=0}^{k-1}{\mfrk{l}(j)}$ in the representation of
$\neg{\varphi(k)}$ are repeated for every $k'>k$.
As shown in~\autoref{sec:res}, the performance results suggest that
the proposed encoding scales well on DLs of realistic size.
Nevertheless, a number of improvements can be envisioned, which add
more structure to the encoding, but with the cost of using additional
auxiliary variables.
Our initial experiments suggest no significant gains were obtained
with a more complex propositional encoding.



\section{Experimental Results} \label{sec:res}
This section aims at assessing the proposed SAT-based approach to
computing and enumerating rigorous abductive explanations
(AXps)~\cite{darwiche-ijcai18,inms-aaai19} as well as contrastive
explanations (CXps)~\cite{inams-aiia20} for decision list models.
First, the approach will be tested from the perspective of raw
performance, followed by additional information on the comparative
number of AXps and CXps as well as their length.

\paragraph{Experimental Setup.}

The experiments were performed on a MacBook Pro laptop running macOS
Big Sur 11.2.3.
Therefore, each individual process was run on a Quad-Core Intel Core
i5-8259U 2.30 GHz processor with 16 GByte of memory.
The memory limit was set to 4 GByte while the time limit used was set
to 1800 seconds, for each individual process to run.

\paragraph{Prototype Implementation.}
A prototype implementation~\footnote{The prototype is available at
\url{https://github.com/alexeyignatiev/xdl-tool}.} of the proposed
approach was developed as a Python script instrumenting incremental
calls to the Glucose~3 SAT solver~\cite{audemard-sat13} using the
PySAT toolkit~\cite{imms-sat18}.
The implementation targets the computation of one explanation (either
an AXp or a CXp) and enumeration of a given number of those, with a
possibility to enumerate all.

It is known~\cite{inams-aiia20} that a CXp can be computed as an MCS
for the encoding formula discussed above and hence CXp enumeration is
implemented in the prototype as LBX-based MCS
enumeration~\cite{mpms-ijcai15}.
Similarly, AXp corresponds to an MUS of the formula and, as a result,
AXp enumeration is done using the MARCO-like MUS enumeration
approach~\cite{pms-aaai13,lm-cpaior13,lpmms-cj16} due to the hitting
set duality between AXps and CXps~\cite{inams-aiia20}.
Concretely, the MARCO-like explainer is organized as two
interconnecting oracles: (i)~a SAT oracle checking (un)satisfiability
of a selected set of clauses of the formula, and (ii)~a minimal
hitting set (MHS) oracle, which computes minimal hitting sets of a
current collection of MCSes of the formula obtained so far.
The MHS oracle was implemented on top of the RC2 MaxSAT solver
exploited incrementally~\cite{imms-jsat19}.
Each iteration of the MARCO-like explainer computes either an AXp or a
CXp.
The former are reported and blocked (by adding a single clause to the
MHS oracle) while the latter are used later as the sets \emph{to hit}.
The explainer stops as soon as there are no more minimal hitting set
identified by the MHS oracle.
As a result, the MARCO-like explainer produces both AXps and CXps upon
the end of execution.
Note that thanks to the use of MaxSAT-based MHS oracle, AXps computed
this way are irredundant, i.e.\ subset-minimal, and do not have to be
reduced further while CXps do need to be reduced by a dedicated
reduction procedure (see below).
Also note that the MARCO-like approach can also be used in a
\emph{dual way}, i.e.\ targeting CXp enumeration and computing AXps as
a by-product.
This mode of operation of the explainer has also been implemented in
the developed prototype.

It is also important to mention that all the three modes of operation
make incremental use of the underlying SAT oracles.
As such, the LBX-like CXp enumeration computes an explanation, blocks
it by adding a single clause and proceeds to the next CXp.
Furthermore, once all explanations for a given data instance are
enumerated, all the previously added blocking clauses are
\emph{disabled} and the enumeration process starts again for a new
data instance.
This is done with the use of unique \emph{selector} variables
introduced for each data instance.
On the contrary, the MARCO-like approaches accumulate and block all
explanations on the MHS oracle side.
This enables one to keep the same SAT oracle on the checking side of
the approach while restarting the MHS oracle from scratch, i.e.\ with
an empty collection of sets to hit, for each new data instance.

Finally, the following heuristics are used.
LBX-like computation of a single CXp makes use of the \emph{Clause D}
(CLD) heuristic~\cite{mshjpb-ijcai13}.
Computation of a single AXp is done as a simple deletion-based linear
search procedure~\cite{msl-sat11}, strengthened by exhaustive
enumeration of unit-size MCSes used to bootstrap the MHS oracle.
Although a more sophisticated algorithm
QuickXPlain~\cite{junker-aaai04} has been also implemented, it turned
out to be outperformed by the aforementioned simpler alternative in
this concrete setting.

\paragraph{Benchmarks and Methodology.}
Experimental evaluation was performed on a subset of datasets selected
from a few publicly available sources.
In particular, these include datasets from UCI Machine Learning
Repository~\cite{uci} and Penn Machine Learning
Benchmarks~\cite{pennml} as well as datasets previously studied in the
context of ML explainability~\cite{guestrin-aaai18} and
fairness~\cite{fairml17,fairness15}.
The number of selected datasets is 72.
%
%
We applied the approach of 5-fold cross validation, i.e.\ each dataset
was randomly split into 5 chunks of instances; each of these chunks
served as test data while the remaining 4 chunks were used to train
the classifiers.
As a result, each dataset (out of 72) resulted in 5 individual pairs
of training and test datasets represented by 80\% and 20\% of data
instances.
Therefore, the total number of training datasets considered in the
evaluation is 360.

Given a training dataset, i.e. represented by 4 chunks of the original
data, a decision list model was trained with the use of the well-known
heuristic algorithm CN2~\cite{clark-ml89,clark-ewsl91}\footnote{Recent
  alternative approaches to \emph{sparse} decision
  lists~\cite{rudin-kdd17,rudin-jmlr17,rudin-mpc18} have also been
  considered but were eventually discarded for two reasons: (1)~they
  can only deal with binary data and (2)~they produce sparse decision
lists containing a couple of rules and a few literals in total ---
i.e.\ these methods do not provide models that would be of interest
for our work.}, the implementation of which was taken from the
well-known Python toolkit
Orange\footnote{\url{https://orangedatamining.com/}}.
The time spent on training the models was ignored.
Next, the prototype explainer was run in one of the three modes
described above, to enumerate \emph{all} explanations (either AXps, or
CXps) for each of the instances of the original 100\% data.
Also and as mentioned above, the explainer was given 1800 seconds for
each of the 360 datasets/models.

Note that the number of rules in the decision list models constructed
by CN2 for the target datasets varied from 6 to 2055.
Also, the total number of non-class, i.e.\ solely antecedent, literals
used in the models varied from 6 to 6754.
Finally, propositional formulas encoding the explanation problems for
these models had from 7 to 15340 variables and from 9 to 3932987
clauses.
It is important to mention that all data was treated as categorical
and hence the propositional formulas given to the encoder incorporated
cardinality constraints enforcing that a feature can take exactly one
value; in the experiments, these constraints were encoded into CNF
using the pairwise encoding~\cite{prestwich-chapt21}.
Although left untested, other cardinality encodings would result in
smaller formulas --- the pairwise encoding was selected intentionally
in order to produce larger formulas and so to test scalability of the
proposed SAT-based approach.

\begin{figure*}[!t]
  \begin{subfigure}[b]{0.49\textwidth}
    \centering
    \includegraphics[width=\textwidth]{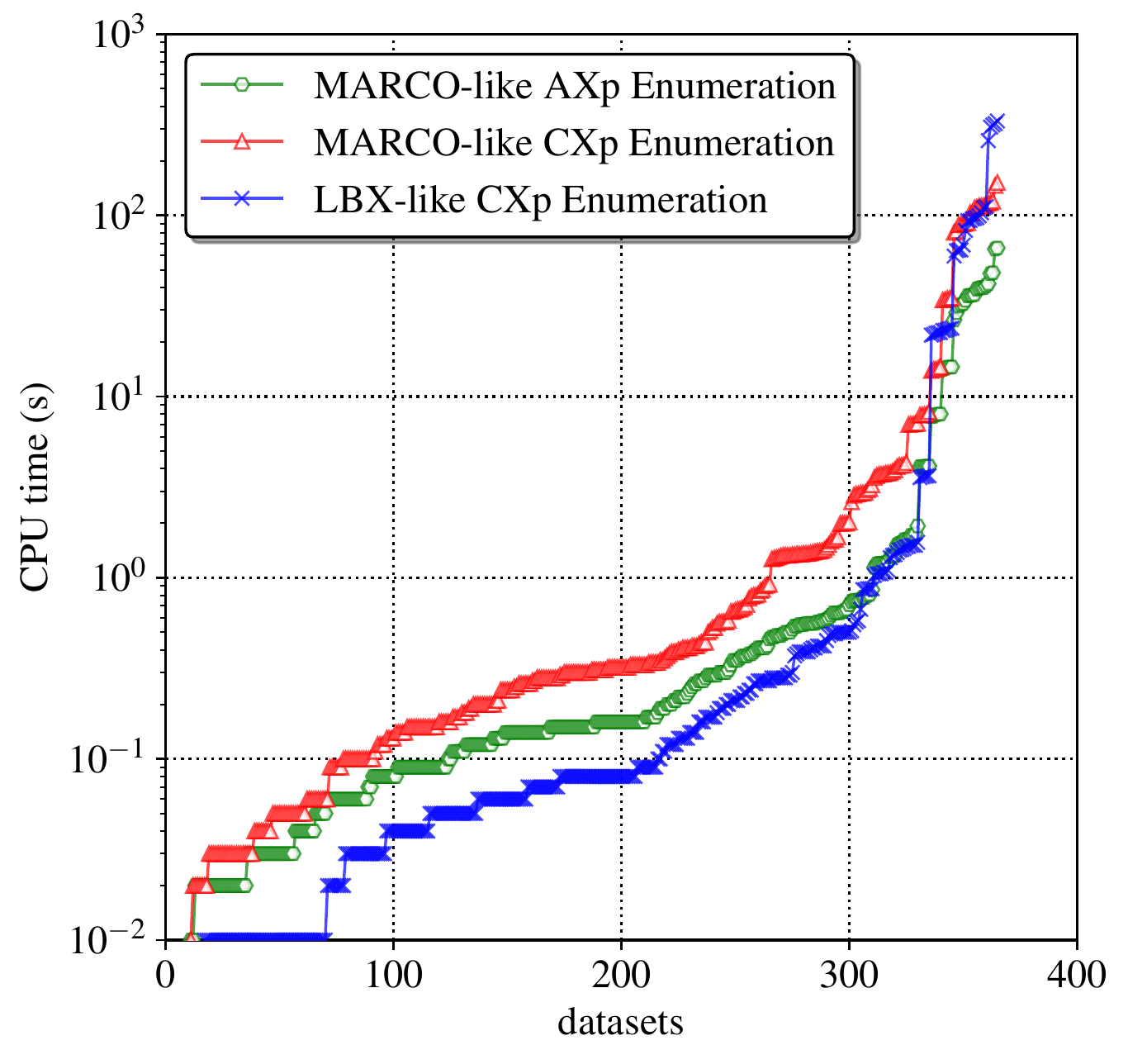}
    \caption{Raw performance comparison}
    \label{fig:cactus-perf}
  \end{subfigure}%
  \begin{subfigure}[b]{0.49\textwidth}
    \centering
    \includegraphics[width=\textwidth]{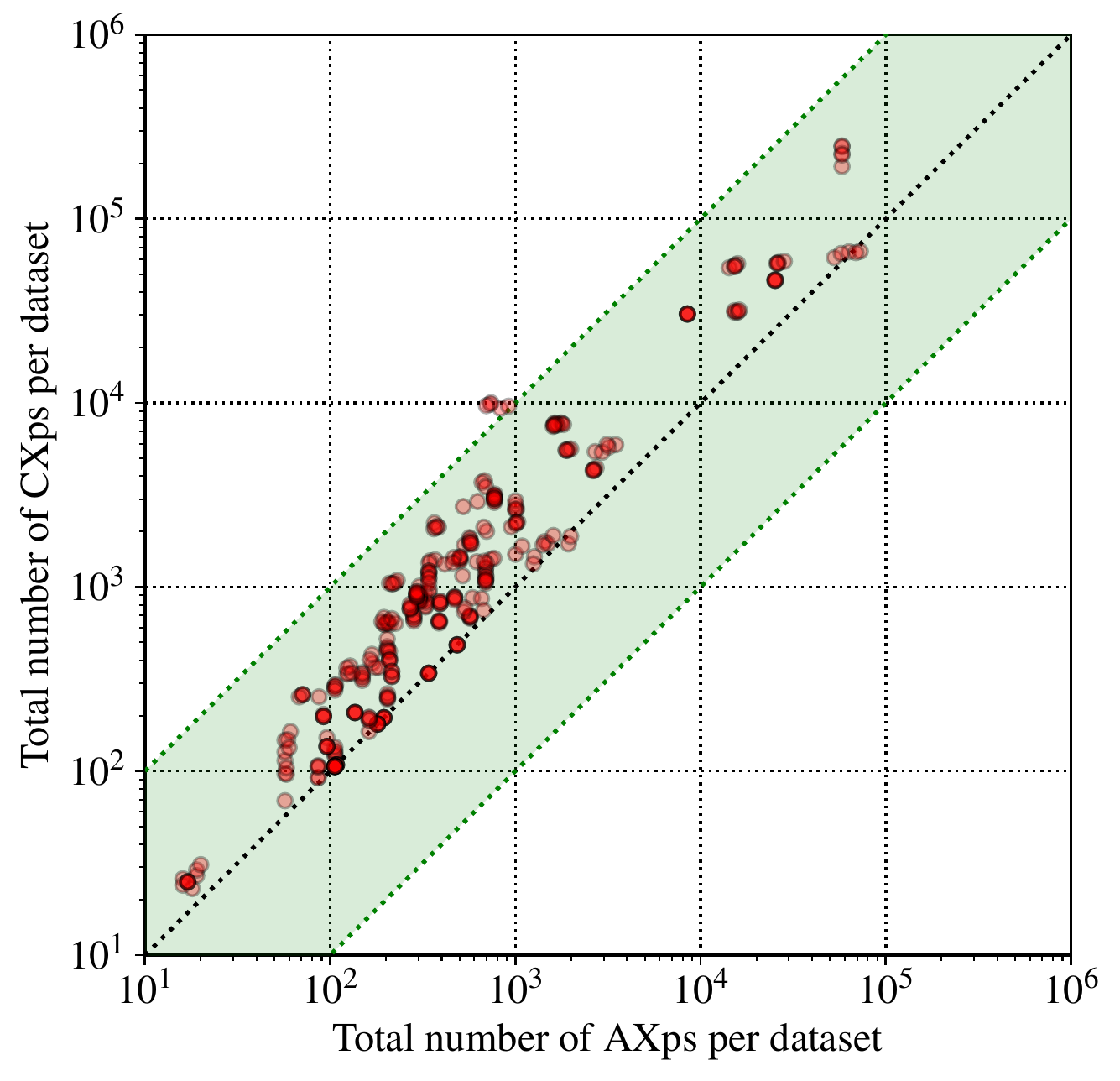}
    \caption{Total number AXps and CXps of per dataset}
    \label{fig:scatter-noftot}
  \end{subfigure}
  \caption{Performance of the three operation modes and the total
  number of explanations per dataset they enumerate.}
  \label{fig:perf}
\end{figure*}

\paragraph{Raw Performance.}

\autoref{fig:cactus-perf} depicts a cactus plot showing the raw
performance of the explainer working in the three selected modes of
operation.
(Note that the CPU time axis is scaled logarithmically.)
As can be observed, all the algorithms are able to finish successful
computation of all the target explanations for all the data instances
of each of the 360 benchmark datasets within the given time limit.
Surprisingly, the best performing configuration overall turns out to
be MARCO-based AXp enumeration.
MARCO-based CXp enumeration is a bit slower.
Recall that both MARCO-based modes end up enumerating the same sets of
explanations including AXps (CXps, resp.) and dual CXps (dual AXps,
resp.).
Also, recall that the only major difference between the two
configurations is the type of the target explanations that are
provided by the MHS oracle while the dual explanations have to be
reduced by a dedicated reduction procedure.
Therefore, the performance difference shown suggests that in practice
it may be more beneficial to target AXps and so to reduce dual CXps
than doing the opposite (which is not really surprising given that the
former correspond to MUS extraction while the latter correspond to MCS
extraction).
Finally, it should be mentioned that although LBX-like CXp explanation
works the most efficiently for most of the benchmarks, in some cases
it is outperformed by the competitors, which may be explained by the
need to incrementally block a significant number of previously
computed solutions (recall that, on the contrary, the MARCO-like
configurations restart the MHS oracle from scratch for every new data
instance).

\begin{figure*}[!t]
  \begin{subfigure}[b]{0.49\textwidth}
    \centering
    \includegraphics[width=\textwidth]{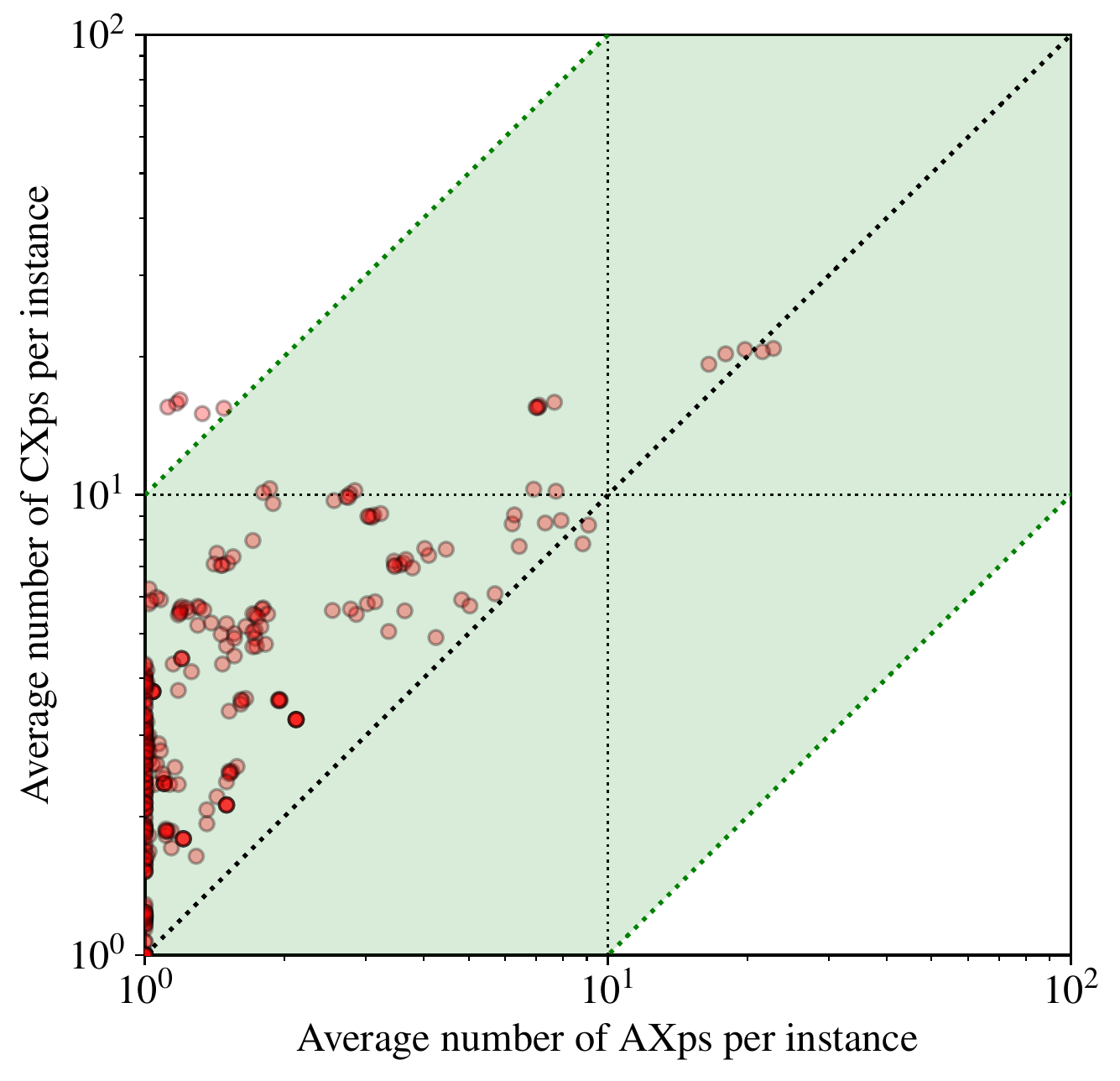}
    \caption{Average number of explanations per instance}
    \label{fig:scatter-enum}
  \end{subfigure}%
  \begin{subfigure}[b]{0.49\textwidth}
    \centering
    \includegraphics[width=\textwidth]{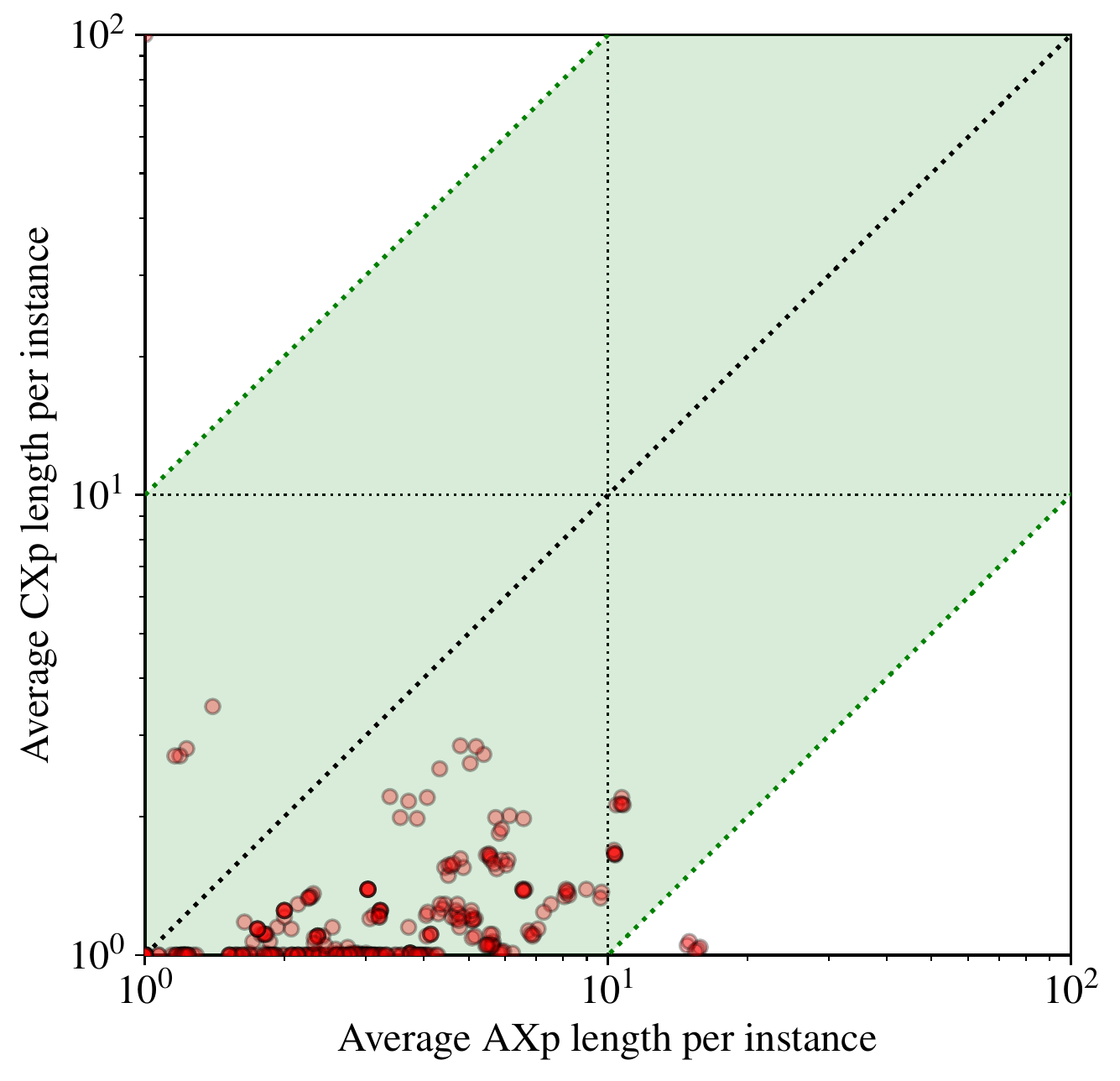}
    \caption{Average explanation size per instance}
    \label{fig:scatter-esize}
  \end{subfigure}
  \caption{Average number of AXps and CXps per data instance and their
  average size.}
  \label{fig:stats}
\end{figure*}

\paragraph{AXps vs CXps.}

As can be seen in~\autoref{fig:scatter-noftot}, the total number of
AXps per dataset tends to be lower than the total number of CXps.
Concretely, the number of AXps per dataset varies from 16 to 72838
while the number of CXps per dataset varies from 23 to 248825.
(Observe that the time to compute one explanation is negligible.)
These data are in line with the results previously obtained
in~\cite{inams-aiia20} when explaining a different kind of ML model
(namely, XGBoost models~\cite{guestrin-kdd16a}) with a different
reasoning engine (namely, Z3 SMT solver~\cite{bjorner-tacas08}).
Unsurprisingly, the average number of CXps per data instance is also
higher than the average number of AXps, as shown
in~\autoref{fig:scatter-enum}.
In general, the average number of CXps per instance varies from 1 to
20.8 while the average number of AXps goes from 1 to 22.7.
However and as one can observe in the scatter
plot~\autoref{fig:scatter-enum}, for the lion's share of data
instances there is a single AXp while there are many more CXps.
Note that the picture is the opposite for the average explanation
length (measured as the number of literals remaining in the
explanation).
In particular, CXps are shorter than AXps and the average length of a
CXp per data instance does not exceed 2.8 while the average length of
AXp varies from 1 to 15.8 (which in fact may provide another insight
into the underperforming MARCO-like CXp enumeration).
Observe that these data also confirms the results previously reported
in~\cite{inams-aiia20}.

\paragraph{Final Remarks.}

A few conclusions can be made with respect to the experimental results
shown above.
First, all the explainer configurations scale well and are able to
enumerate all explanations for all data instances incrementally, even
for DL models with thousands of rules and literals encoded into CNF
formulas with millions of clauses.
Second, MARCO-like AXp enumeration outperforms both LBX-like and
MARCO-like CXp enumeration.
Third, the number of CXps per dataset and per instance tends to be
higher than the number of AXps.
And finally, AXps are on average much larger than CXps.



\section{Conclusions} \label{sec:conc}

This paper investigates the computation of rigorous (or PI-)
explanations for DLs.
The paper first argues that, similar to DTs~\cite{iims-corr20}, DLs
may also not be interpretable. (This observation is also validated by
the experimental results.)
Furthermore, the paper proves that in contrast to the case of DTs,
finding one PI-explanation for DLs (and also for DSs) cannot be in
$\tn{P}$ unless $\tn{P}=\tn{NP}$.
As a result, one possible solution for finding AXps and CXps is to
encode the problem to propositional logic, and find one or enumerate
more than one explanation(s) using SAT oracles.
The experimental results demonstrate that SAT-based approaches are
effective at finding explanations (both AXps and CXps) of DLs. The
experimental results also confirm that a MARCO-like algorithm is
effective at enumerating explanations of DLs.

The results in this paper suggest a number of future research topics.
The application of SAT to explaining DLs motivates the investigation
of which other ML models can be explained with SAT solvers, and for
which explanations can be computed efficiently.
%


\clearpage
\bibliographystyle{splncs04}
\bibliography{refs}

\end{document}